\documentclass[11pt]{article}
\usepackage[T1]{fontenc}
\usepackage{fullpage}
\usepackage{amsmath,amsthm,amsfonts,amssymb,dsfont}
\usepackage{microtype}
\usepackage{hyperref,cleveref}

\Crefformat{equation}{#2(#1)#3}
\crefformat{equation}{#2(#1)#3}
\Crefrangeformat{equation}{#3(#1)#4--#5(#2)#6}
\crefrangeformat{equation}{#3(#1)#4--#5(#2)#6}
\Crefmultiformat{equation}{#2(#1)#3}{ and~#2(#1)#3}{, #2(#1)#3}{ and~#2(#1)#3}
\crefmultiformat{equation}{#2(#1)#3}{ and~#2(#1)#3}{, #2(#1)#3}{ and~#2(#1)#3}

\usepackage{enumitem}
\setlist[description]{font=\normalfont}

\usepackage{csquotes}
\MakeOuterQuote{"}

\usepackage{mathtools}
\mathtoolsset{centercolon}
\DeclarePairedDelimiterXPP\ind[1]{\mathds{1}}{\lbrace}{\rbrace}{}{#1} %
\DeclarePairedDelimiterX\eval[1]{\lbrace}{\rvert}{#1 \delimsize\rbrace} %
\DeclarePairedDelimiter\ip{\langle}{\rangle} %
\DeclarePairedDelimiter\abs{\lvert}{\rvert} %
\DeclarePairedDelimiter\norm{\lVert}{\rVert} %
\DeclarePairedDelimiter\del{\lparen}{\rparen} %
\DeclarePairedDelimiter\set{\lbrace}{\rbrace} %
\DeclarePairedDelimiter\intoc{\lparen}{\rbrack} %
\DeclarePairedDelimiter\intcc{\lbrack}{\rbrack} %

\providecommand\given{}
\newcommand{\pipeseparator}{\nonscript\:\delimsize\vert\nonscript\:\mathopen{}}
\begingroup
  \catcode`\|=13
  \gdef|{\pipeseparator}
\endgroup
\newcommand{\activatepipe}{%
  \renewcommand\given\pipeseparator
  \mathcode`\|="8000
}

\DeclarePairedDelimiterX{\Set}[1]{\{} {\}}{%
  \activatepipe
  #1
}
\DeclarePairedDelimiterX\Sbr[1]\lbrack\rbrack{%
  \activatepipe
  #1
}
\DeclarePairedDelimiterX\Del[1]\lparen\rparen{%
  \activatepipe
  #1
}
\DeclarePairedDelimiterX\Braket[1]\langle\rangle{%
  \activatepipe
  #1
}

\newtheorem{theorem}{Theorem} %
\newtheorem{lemma}{Lemma} %
\newtheorem{corollary}{Corollary} %
\theoremstyle{definition}
\newtheorem{definition}{Definition} %
\newtheorem{remark}{Remark} %

\DeclareMathOperator\sign{sign} %

\newcommand\T{{\scriptscriptstyle{\mathsf{T}}}} %
\newcommand\R{\mathbb{R}} %
\newcommand\N{\mathbb{N}} %

\def\ddefloop#1{\ifx\ddefloop#1\else\ddef{#1}\expandafter\ddefloop\fi}
\def\ddef#1{\expandafter\def\csname bf#1\endcsname{\ensuremath{\mathbf{#1}}}}
\ddefloop abcdefghijklmnopqrstuvwxyzABCDEFGHIJKLMNOPQRSTUVWXYZ\ddefloop
\def\ddef#1{\expandafter\def\csname bf#1\endcsname{\ensuremath{\boldsymbol{\csname #1\endcsname}}}}
\ddefloop {alpha}{beta}{gamma}{delta}{epsilon}{varepsilon}{zeta}{eta}{theta}{vartheta}{iota}{kappa}{lambda}{mu}{nu}{xi}{pi}{varpi}{rho}{varrho}{sigma}{varsigma}{tau}{upsilon}{phi}{varphi}{chi}{psi}{omega}{Gamma}{Delta}{Theta}{Lambda}{Xi}{Pi}{Sigma}{varSigma}{Upsilon}{Phi}{Psi}{Omega}{ell}\ddefloop
\def\ddef#1{\expandafter\def\csname cal#1\endcsname{\ensuremath{\mathcal{#1}}}}
\ddefloop ABCDEFGHIJKLMNOPQRSTUVWXYZ\ddefloop

\usepackage[round]{natbib}
\newcommand\queryemb{\ensuremath{\mathsf{Q}}}
\newcommand\keyemb{\ensuremath{\mathsf{K}}}
\newcommand\valueemb{\ensuremath{\mathsf{V}}}
\newcommand\ips{\ensuremath{\mathbb{V}}}

\newcommand\relu{\ensuremath{\sigma_{\operatorname{r}}}}

\title{Lower bounds for one-layer transformers that compute parity}
\author{Daniel Hsu}

\begin{document}

\maketitle

\begin{abstract}
  This note shows that no self-attention layer post-processed by a rational function can sign-represent the parity function unless the product of the number of heads and the degree of the post-processing function grows linearly with the input length.
  Combining this lower bound with rational approximation of ReLU networks yields a margin-dependent extension for self-attention layers post-processed by ReLU networks.
\end{abstract}

\section{Introduction}

\citet{kozachinskiy2026parity} recently showed that no self-attention layer post-processed by a fixed-size ReLU network can compute the parity of $n$ input bits, for sufficiently large $n$.
Their proof uses a result of \citet{kane2013correct} and first-order logical representations to establish a $\sqrt{n}(\log n)^{O(1)}$ upper bound on the average sensitivity of such self-attention layers, so they cannot compute functions of higher average sensitivity such as parity.

If the ReLU network was permitted to have "complexity" growing with $n$, then it is easy to construct a self-attention head---and the post-processing ReLU network---that computes parity.
So the insufficiency of self-attention layers is sensitive to the post-processing function.
This note gives a quantitative variant of \citeauthor{kozachinskiy2026parity}'s result that considers the complexity of the post-processing function.
For rational post-processing, the lower bound below shows that increasing only the value-embedding dimension does not help: the product of the number of heads and the rational degree must be large.
Combining the lower bound with rational approximation of ReLU networks gives a margin-based lower bound for self-attention layers post-processed by ReLU networks.

\section{Definitions}

\begin{definition} \label{def:sa}
  A \emph{self-attention head $A \colon \set{0,1}^n \to \R^d$ with value-embedding dimension $d$ (for binary inputs)} is a function of the form
  \begin{equation*}
    A(x_1,\dotsc,x_n)
    =
    \frac{
      \sum_{i=1}^n \exp\del{ \ip{\queryemb(x_n), \keyemb_i(x_i)} } \, \valueemb_i(x_i)
    }{
      \sum_{i=1}^n \exp\del{ \ip{\queryemb(x_n), \keyemb_i(x_i)} }
    }
    ,
  \end{equation*}
  where $\queryemb \colon \set{0,1} \to \ips$, $\keyemb_i \colon \set{0,1} \to \ips$, and $\valueemb_i \colon \set{0,1} \to \R^d$ are arbitrary (position-specific) embedding functions, $(\ips,\ip{\cdot,\cdot})$ is an inner product space, and $\R^d$ is the $d$-dimensional Euclidean space with coordinate basis vectors $e_1,\dotsc,e_d$.
  (We only consider the output corresponding to $x_n$.)
\end{definition}

\begin{definition} \label{def:ratdeg}
  A \emph{rational function of degree at most $p$} on $\R^d$ is a function $u=P/Q$, where $P,Q \in \R[z_1,\dotsc,z_d]$ have total degree at most $p$ and $Q$ is nonzero on the domain where $u$ is evaluated.
  More generally, for a rational function $R=P/Q$ on $\set{0,1}^n$, with $Q$ nonzero on $\set{0,1}^n$, we write
  \begin{equation*}
    \deg(R) \coloneqq \max\set{\deg(P),\deg(Q)} .
  \end{equation*}
\end{definition}

\begin{definition} \label{def:sal}
  A \emph{self-attention layer with $h$ heads and value-embedding dimension $d$ post-processed by a function $u \colon \R^d \to \R$} is the composition of $u$ with the sum of $h$ self-attention heads $A_1,\dotsc,A_h$ each with value-embedding dimension $d$
  \begin{equation*}
    x \mapsto u\del*{ A_1(x) + \dotsb + A_h(x) } .
  \end{equation*}
  When $u$ is rational, we require its denominator to be nonzero on every vector $A_1(x)+\dotsb+A_h(x)$ with $x \in \set{0,1}^n$.
\end{definition}

\begin{definition} \label{def:sgn}
  A real-valued function $f \colon \set{0,1}^n \to \R$ \emph{sign-represents} a Boolean function $g \colon \set{0,1}^n \to \set{0,1}$ if, for every $x \in \set{0,1}^n$,
  \begin{equation*}
    f(x) > 0 \iff g(x) = 1 .
  \end{equation*}
\end{definition}

\begin{definition} \label{def:margin}
  A real-valued function $f \colon \set{0,1}^n \to \R$ \emph{$(\tau,\gamma)$-represents} a Boolean function $g \colon \set{0,1}^n \to \set{0,1}$ if, for every $x \in \set{0,1}^n$,
  \begin{equation*}
    g(x)=1 \implies f(x) \geq \tau+\gamma,
    \qquad
    g(x)=0 \implies f(x) \leq \tau-\gamma .
  \end{equation*}
  Equivalently, $f-\tau$ sign-represents $g$ with margin $\gamma$.
\end{definition}

\begin{definition} \label{def:relu}
  Let $\relu(t) \coloneqq \max\set{0,t}$.
  A \emph{normalized ReLU network of width $m$ and depth $\ell$} is a scalar-output feed-forward network with at most $m$ ReLU gates in each of at most $\ell$ layers, where each gate computes $z \mapsto \relu(a^\T z+b)$ with $\norm{a}_1+\abs{b} \leq 1$.
\end{definition}

\section{Lower bound}
\label{sec:lb}

\begin{theorem} \label{thm:parity}
  Let $T$ be a self-attention layer with $h$ heads and value-embedding dimension $d$ post-processed by a rational function of degree at most $p$.
  If $T$ sign-represents the parity function on $n$ input bits, then
  \begin{equation*}
    hp \geq \frac{n}{4} .
  \end{equation*}
\end{theorem}

The proof of \Cref{thm:parity} uses a rational representation of self-attention heads over $\set{0,1}^n$ inspired by the proof of the average sensitivity bound of \citet{kozachinskiy2026parity}.

\begin{lemma} \label{lem:coord}
  Let $A \colon \set{0,1}^n \to \R^d$ be a self-attention head as in \Cref{def:sa}.
  For every coordinate $k \in [d]$, there are polynomials $N_k,D \in \R[x_1,\dotsc,x_n]$ of degree at most $2$ such that $D(x)>0$ and
  \begin{equation*}
    e_k^\T A(x) = \frac{N_k(x)}{D(x)}
  \end{equation*}
  for all $x \in \set{0,1}^n$.
\end{lemma}

\begin{proof}
  For any $k \in [d]$ and $i \in [n]$,
  \begin{equation*}
    \exp\del{ \ip{\queryemb(x_n), \keyemb_i(x_i)} } \, e_k^\T \valueemb_i(x_i)
    =
    \sum_{(a,b) \in \set{0,1}^2}
    \exp\del{ \ip{\queryemb(a), \keyemb_i(b)} } \, e_k^\T \valueemb_i(b) \, \ind{(x_n,x_i) = (a,b)} .
  \end{equation*}
  On $\set{0,1}^2$,
  \begin{equation*}
    \ind{(x_n,x_i) = (a,b)} = \del*{ 1-a + (2a-1)x_n } \del*{ 1-b + (2b-1)x_i } ,
  \end{equation*}
  which is a polynomial in $x_n$ and $x_i$ of degree at most $2$.
  Hence each summand in the numerator of $e_k^\T A$ has degree at most $2$, and so the numerator
  \begin{equation*}
    N_k(x) \coloneqq \sum_{i=1}^n
    \exp\del{ \ip{\queryemb(x_n), \keyemb_i(x_i)} } \, e_k^\T \valueemb_i(x_i)
  \end{equation*}
  has a degree-at-most-$2$ polynomial representation on $\set{0,1}^n$.
  Similarly,
  \begin{equation*}
    D(x) \coloneqq \sum_{i=1}^n \exp\del{ \ip{\queryemb(x_n), \keyemb_i(x_i)} }
  \end{equation*}
  has a degree-at-most-$2$ polynomial representation on $\set{0,1}^n$.
  Since $D$ is a sum of positive exponentials, $D(x)>0$ for every $x \in \set{0,1}^n$.
\end{proof}

\begin{lemma} \label{lem:layer-degree}
  Let $T$ be a self-attention layer with $h$ heads and value-embedding dimension $d$ post-processed by a rational function of degree at most $p$.
  Then $T$, restricted to $\set{0,1}^n$, is a rational function of degree at most $2hp$.
\end{lemma}

\begin{proof}
  Write the $j$th head as $A_j=(N_{j,1}/D_j,\dotsc,N_{j,d}/D_j)$ using \Cref{lem:coord}, where each $N_{j,k}$ and $D_j$ has degree at most $2$ and $D_j>0$ on $\set{0,1}^n$.
  Let
  \begin{equation*}
    S \coloneqq \prod_{j=1}^h D_j .
  \end{equation*}
  Then $S$ has degree at most $2h$ and is positive on $\set{0,1}^n$.
  For every coordinate $k$, the $k$th coordinate of the sum of the heads is
  \begin{equation*}
    \sum_{j=1}^h \frac{N_{j,k}}{D_j}
    =
    \frac{M_k}{S},
    \qquad
    M_k \coloneqq \sum_{j=1}^h N_{j,k} \prod_{\ell \neq j} D_\ell .
  \end{equation*}
  Each $M_k$ has degree at most $2h$.

  Let the post-processing function be $u=P/Q$, where $P,Q \in \R[z_1,\dotsc,z_d]$ have total degree at most $p$, and $Q$ is nonzero on all attained head sums.
  For a multi-index $\alpha \in \N^d$, write $z^\alpha=z_1^{\alpha_1}\dotsm z_d^{\alpha_d}$ and $|\alpha|=\alpha_1+\dotsb+\alpha_d$.
  If $P(z)=\sum_{|\alpha|\leq p} c_\alpha z^\alpha$, define
  \begin{equation*}
    \widetilde P(x)
    \coloneqq
    \sum_{|\alpha|\leq p} c_\alpha
    \del*{\prod_{k=1}^d M_k(x)^{\alpha_k}} S(x)^{p-|\alpha|} .
  \end{equation*}
  Since $\deg(M_k)\leq 2h$ and $\deg(S)\leq 2h$, every term in $\widetilde P$ has degree at most $2hp$.
  Moreover,
  \begin{equation*}
    P\del*{\frac{M_1}{S},\dotsc,\frac{M_d}{S}} = \frac{\widetilde P}{S^p}
  \end{equation*}
  on $\set{0,1}^n$.
  Defining $\widetilde Q$ analogously gives $\deg(\widetilde Q)\leq 2hp$ and
  \begin{equation*}
    T(x)=u\del*{A_1(x)+\dotsb+A_h(x)}=\frac{\widetilde P(x)}{\widetilde Q(x)}
  \end{equation*}
  on $\set{0,1}^n$.
  The denominator $\widetilde Q$ is nonzero on $\set{0,1}^n$ because $S>0$ there and $Q$ is nonzero on every attained head sum.
\end{proof}

\begin{proof}[Proof of \Cref{thm:parity}]
  By \Cref{lem:layer-degree}, $T$ is a rational function of degree at most $2hp$ on $\set{0,1}^n$.
  By Corollary~1 of \citet{paturi1994approximating}, if a rational function $P(x)/Q(x)$ sign-represents the parity function on $n$ input bits, then $\deg(P) + \deg(Q) \geq n$.
  Therefore any rational function sign-representing parity has degree at least $n/2$.
  Hence, if $T$ sign-represents parity, then
  \begin{equation*}
    2hp \geq \deg(T) \geq \frac{n}{2} .
    \qedhere
  \end{equation*}
\end{proof}

\begin{remark}
  Since the sign of a rational function $P(x)/Q(x)$ is equivalent to the polynomial threshold function $\sign(P(x)Q(x))$, combining \Cref{lem:layer-degree} with the average sensitivity bounds of \citet{kane2013correct} implies an upper bound on the average sensitivity of $\sign(T)$ (for $T$ as in \Cref{lem:layer-degree}) of
  \begin{equation*}
    \sqrt{n} (\log n)^{O(hp \log (hp))} 2^{O((hp)^2 \log(hp))} .
  \end{equation*}
  This recovers an analogue of the bound of \citet{kozachinskiy2026parity} for rational function post-processing showing explicit dependence on $h$ and $p$.
  By standard techniques in Boolean function analysis, we obtain an upper bound on the correlation of $\sign(T)$ with the parity function of
  \begin{equation*}
    n^{-1/4} (\log n)^{O(hp \log (hp))} 2^{O((hp)^2 \log(hp))} .
  \end{equation*}
  If the eponymous (asymptotic) conjecture of~\citet{gotsman1994spectral} holds, then the correlation bound improves to $O(n^{-1/4} \sqrt{hp})$.
\end{remark}

\section{ReLU network post-processing}
\label{sec:relu}

\Cref{sec:lb} gives a lower bound against self-attention layers post-processed by rational functions.
To handle post-processing by other types of functions, we can use rational function approximation.
For instance, a theorem of \citet{newman1964rational} can be used to obtain bounds on the degrees of the best rational function approximations to threshold circuits~\citep{paturi1994approximating} and ReLU networks~\citep{telgarsky2017neural}.
Then \Cref{thm:parity} may give lower bounds on the number of self-attention heads and relevant properties of the post-processing function.
In this \namecref{sec:relu}, we carry out this approach for ReLU network post-processing.

\begin{theorem} \label{thm:relu}
  There is a universal constant $C>0$ such that the following holds.
  Let
  \begin{equation*}
    T(x) = u\del*{A_1(x)+\dotsb+A_h(x)}
  \end{equation*}
  be a self-attention layer with $h$ heads and value-embedding dimension $d$, and assume
  \begin{equation*}
    A_1(x)+\dotsb+A_h(x) \in \intcc{-1,1}^d
    \qquad
    \text{for every } x \in \set{0,1}^n .
  \end{equation*}
  Suppose $u\colon \intcc{-1,1}^d \to \R$ is computed by a normalized ReLU network of width $m$ and depth $\ell$.
  If $T$ $(\tau,\gamma)$-represents the parity function on $n$ input bits for some threshold $\tau \in \R$ and margin $\gamma \in \intoc{0,1}$, then
  \begin{equation*}
    h m^\ell \ln\del*{ \frac{2\ell}{\gamma} }^{2\ell}
    \geq
    \frac{n}{C} .
  \end{equation*}
\end{theorem}

\begin{proof}
  By the rational approximation theorem for ReLU networks of \citet[Theorem~1.1(2)]{telgarsky2017neural}, for every $\varepsilon \in \intoc{0,1}$ there is a rational function $v \colon \intcc{-1,1}^d \to \R$ of degree at most
  \begin{equation*}
    p \leq \frac14 C m^\ell \ln\del*{ \frac{\ell}{\varepsilon} }^{2\ell}
  \end{equation*}
  such that
  \begin{equation*}
    \sup_{y \in \intcc{-1,1}^d} \abs{u(y)-v(y)} \leq \varepsilon .
  \end{equation*}
  Take $\varepsilon=\gamma/2$.
  For each $x \in \set{0,1}^n$, write $Y(x) \coloneqq A_1(x)+\dotsb+A_h(x)$.
  The range assumption gives $Y(x) \in \intcc{-1,1}^d$, so
  \begin{equation*}
    \abs{u(Y(x))-v(Y(x))} \leq \frac{\gamma}{2} .
  \end{equation*}
  If the parity function on $x$ evaluates to $1$, then
  \begin{equation*}
    v(Y(x))-\tau
    \geq u(Y(x))-\tau-\frac{\gamma}{2}
    \geq \frac{\gamma}{2} >0 .
  \end{equation*}
  If the parity function on $x$ evaluates to $0$, then
  \begin{equation*}
    v(Y(x))-\tau
    \leq u(Y(x))-\tau+\frac{\gamma}{2}
    \leq -\frac{\gamma}{2}<0 .
  \end{equation*}
  Therefore $x \mapsto v(Y(x))-\tau$ sign-represents parity.
  Since $v-\tau$ is still a rational function of degree at most $p$, \Cref{thm:parity} implies $hp \geq n/4$.
  Substituting the bound on $p$ with $\varepsilon=\gamma/2$ proves the result.
\end{proof}

\begin{corollary} \label{cor:relu-poly-margin}
  If, in \Cref{thm:relu}, the depth $\ell$ is constant and the margin satisfies $\gamma \geq n^{-O(1)}$, then
  \begin{equation*}
    h m^\ell
    \geq
    \frac{n}{(\log n)^{O(\ell)}} .
  \end{equation*}
  In particular, for constant depth and polynomial margin, a self-attention layer post-processed by a ReLU network needs the product $hm^\ell$ to be nearly linear in $n$.
\end{corollary}

\section*{AI usage}

GPT 5.5 was prompted to derive and write-up the result in \Cref{sec:relu} based on a version of this note without that \namecref{sec:relu} (and associated definitions/discussion).

\bibliography{bib}
\bibliographystyle{plainnat}

\end{document}